\pdfoutput=1
\documentclass[letterpaper]{article} 
\usepackage{aaai18}  
\usepackage{times}  
\usepackage{helvet}  
\usepackage{courier}  
\usepackage{url}  
\usepackage{graphicx}  
\usepackage{amssymb}
\usepackage{amsmath}
\usepackage{amsthm}
\usepackage{color}
\usepackage[toc,page]{appendix}

\newtheorem{theorem}{Theorem}
\newtheorem{lemma}[theorem]{Lemma}

\newtheorem{definition}{Definition}

\newtheorem*{example*}{Example}

\newcommand{\comment}[1]{}

\newcommand{\smallmath}[1]{\mbox{\scriptsize ${#1}$}}

\newcommand{\mathtext}[1]{\mathit{#1}}

\newcommand{\cnL}{\mbox{${\cal L}$}}                       
\newcommand{\cnSs}{\mbox{${\cal S}$}}                      
\newcommand{\cnAs}{\mbox{${\cal A}$}}                      
\newcommand{\cnTs}{\mbox{${\cal T}$}}                      
\newcommand{\cnOs}{\mbox{${\cal O}$}}                      
\newcommand{\cnCs}{\mbox{${\cal C}$}}                      
\newcommand{\cnN}{\mbox{${N}$}}                            
\newcommand{\cnPs}{\mbox{$\Pi$}}                           
\newcommand{\cnPSF}{\mbox{$\lambda$}}                      
\newcommand{\cnOUO}{\mbox{$\tau$}}                         
\newcommand{\cnAUO}{\mbox{$\gamma$}}                       

\newcommand{\cnS}{\mbox{$s$}}                          
\newcommand{\cnP}{\mbox{$\pi$}}                        
\newcommand{\cnAss}{\mbox{${a}$}}                      

\newcommand{\cnH}{\mbox{$H$}}                   

\newcommand{\cnNs}{\mbox{${\cal N}$}}                  
\newcommand{\cnFPs}{\mbox{${F}$}}                      

\newcommand{\cnAN}{\mbox{$Q$}}                         
\newcommand{\cnAH}{\mbox{$\cal X$}}                    
\newcommand{\cnANs}{\mbox{$\cal Q$}}                   

\newcommand{\cnSF}{\mbox{$\phi$}}                      
\newcommand{\cnTF}{\mbox{$\psi$}}                      
\newcommand{\cnCF}{\mbox{$\varrho$}}                   

\newcommand{\cnGG}{\mbox{$\Psi$}}                      

\newcommand{\cnEF}{\mbox{$f_{\emptyset}$}}             

\newcommand{\cnAPNU}{\mbox{\text{\bf PredUpdate$'$}}}       
\newcommand{\cnAPGU}{\mbox{\text{\bf PredUpdate}}}       

\newcommand{\cnPGU}{\mbox{\text{\bf Update}}}          

\newcommand{\exAH}{\mbox{$/\!\!\!-\!\!\!\backslash$}}

\newcommand{\pearlD}{\mbox{$d$}}
\newcommand{\pearlC}{\mbox{$c$}}

\newcommand{\cnNorm}{\mbox{$\eta$}}  

\title{Perceptual Context in Cognitive Hierarchies}
\author{Bernhard Hengst
\and
Maurice Pagnucco
\and
\vspace{0.5mm}
David Rajaratnam\\
\vspace{1mm}
{\bf \Large Claude Sammut
\and
Michael Thielscher}\\
School of Computer Science and Engineering\\
The University of New South Wales, Australia\\
\{bernhardh,morri,daver,claude,mit\}@cse.unsw.edu.au
}

\comment{
\author{
Bernhard Hengst\\
University of NSW\\ Australia
\And 
Maurice Pagnucco\\
University of NSW\\ Australia
\And
David Rajaratnam\\
University of NSW\\ Australia
\AND
Claude Sammut\\
University of NSW\\ Australia
\And
Michael Thielscher\\
University of NSW\\ Australia
}
}
\nocopyright

\begin{document}

\maketitle

\begin{abstract}
Cognition does not only depend on bottom-up sensor feature abstraction, but also relies on contextual information being passed top-down. Context is higher level information that helps to predict belief states at lower levels. The main contribution of this paper is to provide a formalisation of perceptual context and its integration into a new process model for cognitive hierarchies. Several simple instantiations of a cognitive hierarchy are used to illustrate the role of context. Notably, we demonstrate the use context in a novel approach to visually track the pose of rigid objects with just a 2D camera. 
\end{abstract}


\section{Introduction}

There is strong evidence that intelligence necessarily involves hierarchical structures~\cite{Ahyby52,1087032,dietterich00hierarchical,Albus-2001,beer66decision,turchin77phenomenon,hubel79brain,Minsky:1986:SM:22939,drescher91made,dayan92feudal,kaelbling93hierarchical,Nilsson01,Konidaris05122011,jong2010structured,Marthi06,Bakker04hierarchicalreinforcement}.
\citeauthor{clark2016framework}~(\citeyear{clark2016framework}) recently have addressed the formalisation of cognitive hierarchies that allow for the integration of disparate representations, including symbolic and sub-symbolic representations, in a framework for cognitive robotics.  Sensory information processing is upward-feeding, progressively abstracting more complex state features, while behaviours are downward-feeding progressively becoming more concrete, ultimately controlling robot actuators.



However, neuroscience suggests that the brain is subject to top-down cognitive influences for attention, expectation and perception \cite{Gilbert2013Top-down-influe}. Higher level signals carry important information to facilitate scene interpretation. For example, the recognition of the Dalmatian, and the disambiguation of the symbol $\exAH$ in Figure \ref{figDalmationCat} intuitively show that higher level context is necessary to correctly interpret these images\footnote{Both of these examples appear in~\cite{Johnson:2010:DMM:1875171} but are also well-known in the cognitive psychology literature.}.
\begin{figure}[ht]
	\centering
	\includegraphics[width=0.23\textwidth]{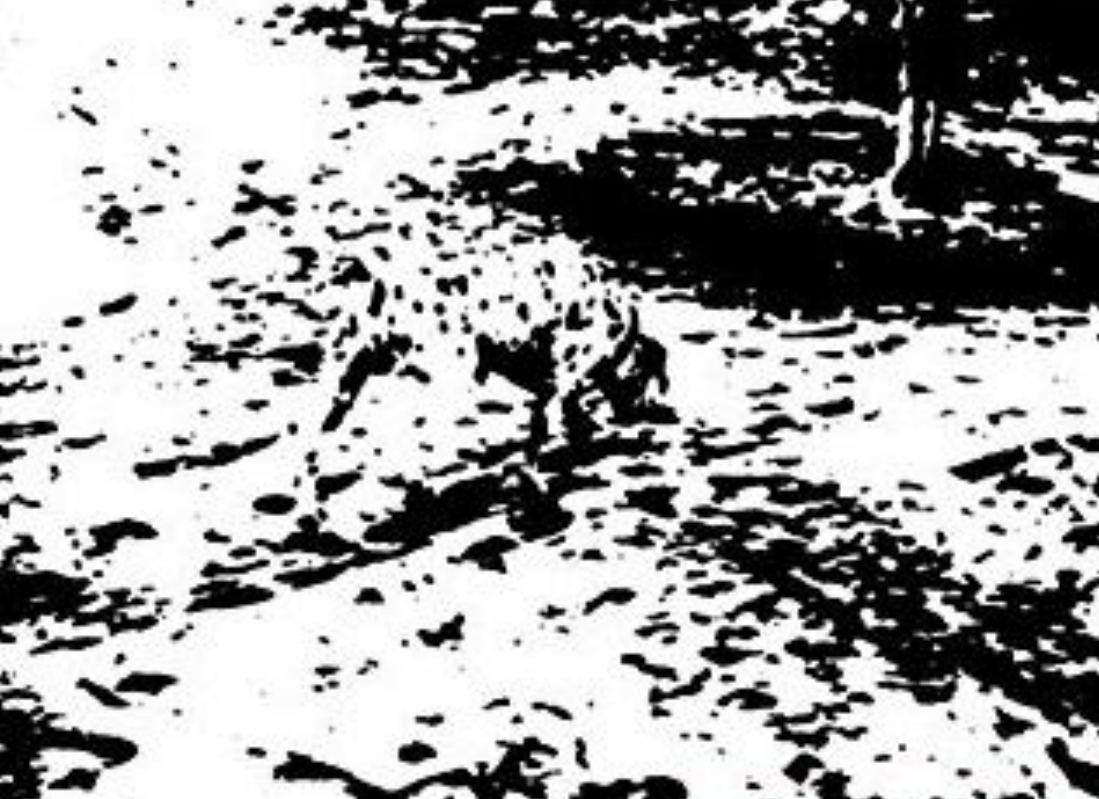}
	\includegraphics[width=0.23\textwidth]{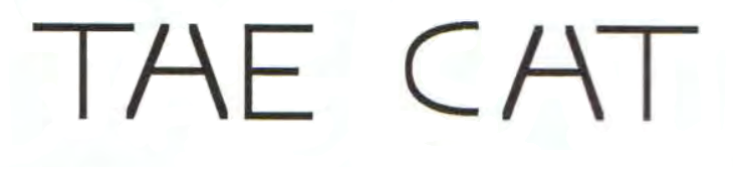}
	\caption{The image on the left would probably be indiscernible without prior knowledge of Dalmations. The ambiguous symbol $\exAH$ on the right can be interpreted as either an ``H'' or an ``A'' depending on the word context.}
	\label{figDalmationCat}
\end{figure}
Furthermore, the human brain is able to make sense of dynamic 3D scenes from light falling on our 2D retina in varying lighting conditions. Replicating this ability is still a challenge in artificial intelligence and computer vision, particularly when objects move relative to each other, can occlude each other, and are without texture. Prior, more abstract contextual knowledge is important to help segment images into objects or to confirm the presence of an object from faint or partial edges in an image.

In this paper we extend the existing cognitive hierarchy formalisation \cite{clark2016framework} by introducing the notion of perceptual context, which modifies the beliefs of a child node given the beliefs of its parent nodes. It is worth emphasising that defining the role of context as a top-down predictive influence on a node's belief state and the corresponding process model that defines how the cognitive hierarchy evolves over time is non-trivial. Our formalisation captures the dual influences of context and behaviour as a predictive update of a node's belief state.
Consequently, \emph{the main contribution of this paper is the inclusion and formalisation of contextual influences as a predictive update within a cognitive hierarchy.}

As a meta-framework, the cognitive hierarchy requires instantiation. We provide two simple instantiation examples to help illustrate the formalisation of context. The first is a running example using a small belief network. The second example involves visual servoing to track a moving object. This second example quantifies the benefit of context and demonstrates the role of context in a complete cognitive hierarchy including behaviour generation.

As a third, realistic and challenging example that highlights the importance of context we consider the tracking of the 6 DoF pose of multiple, possibly occluded, marker-less objects with a 2D camera. We provide a novel instantiation of a cognitive hierarchy for a real robot using the context of a spatial cognitive node modelled using a 3D physics simulator. Note, this formalisation is provided in outline only due to space restrictions.

Finally, for completeness of our belief network running example, we prove that general belief propagation in causal trees  \cite{pearl88probabilistic} can be embedded into our framework, illustrating the versatility of including context in the cognitive hierarchy. We include this proof as an appendix.



\comment{
The existing formal meta-model of cognitive hierarchies~\cite{clark2016framework} does not include a notion of context. In this paper we extend this prior work to include contextual elements, a context function and a revised prediction update process. 
}

\section{The Architectural Framework}

For the sake of brevity the following presentation both summarises and extends the formalisation of cognitive hierarchies as introduced in~\cite{clark2016framework}. We shall, however, highlight how our contribution differs from their work. The essence of this framework is to adopt a meta-theoretic approach, formalising the interaction between abstract cognitive nodes, while making no commitments about the representation and reasoning mechanism within individual nodes.

\subsection{Motivating Example}
\label{s:example}
As an explanatory aid to formalising the use of context in a hierarchy we will use the disambiguation of the symbol $\exAH$ in Figure \ref{figDalmationCat} as a simple running example. This system can be modelled as a two layer causal tree updated according Pearl's Bayesian belief propagation rules~\cite{pearl88probabilistic}. The lower-level layer disambiguates individual letters while the higher-level layer disambiguates complete words (Figure~\ref{figTheCatHierarchy}). We assume that there are only two words that are expected to be seen, with equal probability: ``THE'' and ``CAT''.

\begin{figure}[ht]
	\centering
	\includegraphics[width=0.23\textwidth]{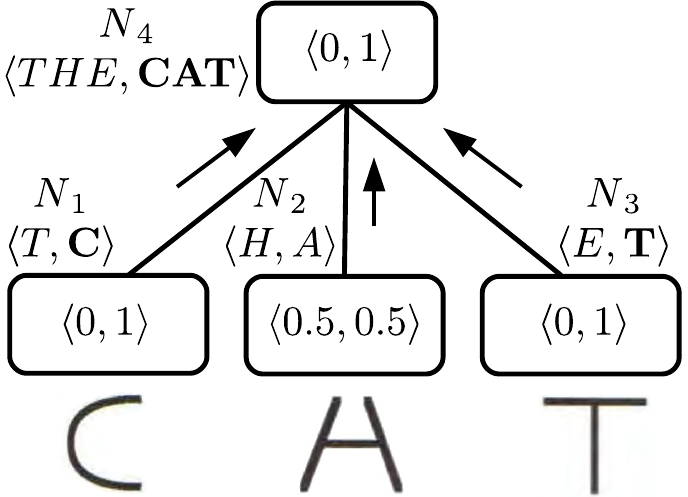}
\hspace{0.3mm}
	\includegraphics[width=0.23\textwidth]{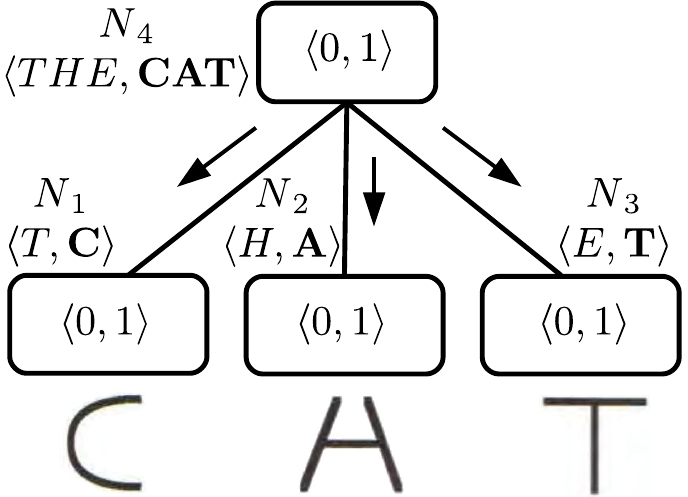}
	\caption{Disambiguating the symbol $\exAH$ requires context from the word recognition layer.}
	\label{figTheCatHierarchy}
\end{figure}

There are three independent letter sensors with the middle sensor being unable to disambiguate the observed symbol $\exAH$ represented by the conditional probabilities $p(\exAH | H) = 0.5$ and $p(\exAH | A) = 0.5$. These sensors feed into the lower-level nodes (or \emph{processors} in Pearl's terminology), which we label as $\cnN_1,\cnN_2,\cnN_3$. The results of the lower level nodes are combined at $\cnN_4$ to disambiguate the observed word.

Each node maintains two state variables; the \emph{diagnostic} and \emph{causal} supports (displayed as the pairs of values in Figure~\ref{figTheCatHierarchy}). Intuitively, the diagnostic support represents the knowledge gathered through sensing while the causal support represents the contextual bias. A node's overall belief is calculated by the combination of these two state variables.

While sensing data propagates up the causal tree, the example highlights how node $\cnN_2$ is only able to resolve the symbol $\exAH$ in the presence of contextual feedback from $\cnN_4$.

\subsection{Nodes}
A cognitive hierarchy consists of a set of nodes. Nodes are tasked to achieve a goal or maximise future value. They have two primary functions: world-modelling and behaviour-generation. World-modelling involves maintaining a \emph{belief state}, while behaviour-generation is achieved through \emph{policies}, where a policy maps states to sets of actions. A node's belief state is modified by sensing or by the combination of actions and higher-level context. We refer to this latter as \emph{prediction update} to highlight how it sets an expectation about what the node is expecting to observe in the future.

\begin{definition}
\label{d:cnode}
A cognitive language is a tuple $\cnL =(\cnSs, \cnAs, \cnTs, \cnOs, \cnCs)$, where $\cnSs$ is a set of belief states, $\cnAs$ is a set of actions, $\cnTs$ is a set of task parameters, $\cnOs$ is a set of observations, and $\cnCs$ is a set of contextual elements. A cognitive node is a tuple $\cnN = (\cnL, \cnPs, \cnPSF, \cnOUO,  \cnAUO, \cnS^0, \cnP^0)$ s.t:
\begin{itemize}
\item $\cnL$ is the cognitive language for $\cnN$, with initial belief state $\cnS^0\in\cnSs$.
\item $\cnPs$ a set of policies such that for all $\cnP \in \cnPs$, $\cnP : \cnSs \rightarrow 2^{\smallmath{\cnAs}}$, with initial policy $\cnP^0 \in \cnPs$.
\item A policy selection function $\cnPSF\!: 2^{\smallmath{\cnTs}} \rightarrow \cnPs$, s.t. $\cnPSF(\{\}) = \cnP^0$.
\item A observation update operator $\cnOUO: 2^{\smallmath{\cnOs}} \times \cnSs \rightarrow \cnSs$.
\item A prediction update operator $\cnAUO: 2^{\smallmath{\cnCs}} \times 2^{\smallmath{\cnAs}} \times \cnSs \rightarrow \cnSs$.
\end{itemize}
\end{definition}

Definition~\ref{d:cnode} differs from \cite{clark2016framework} in two ways: the introduction of a set of context elements in the cognitive language, and the modification of the \emph{prediction} update operator, previously called the \emph{action} update operator, to include context elements when updating the belief state.

This definition can now be applied to the motivating example to instantiate the nodes in the Bayesian causal tree. We highlight only the salient features for this instantiation.

\begin{example*}
Let $E=\{\langle x,y\rangle~|~0\!\le\!x,y\leq\!1.0\}$ be the set of probability pairs, representing the recognition between two distinct features. For node $\cnN_2$, say (cf.\ Figure~\ref{figTheCatHierarchy}), these features are the letters ``H'' and ``A'' and for $\cnN_4$ these are the words ``THE'' and ``CAT''. The set of belief states for $\cnN_2$ is $\cnSs_2 = \{ \langle\langle\pearlD\rangle,\pearlC\rangle ~|~  \pearlD,\pearlC \in E\}$, where $\pearlD$ is the \emph{diagnostic} support and $\pearlC$ is the \emph{causal} support. Note, the vector-in-vector format allows for structural uniformity across nodes. Assuming equal probability over letters, the initial belief state is $\langle\langle\langle 0.5,0.5\rangle\rangle, \langle 0.5,0.5 \rangle\rangle$. For $\cnN_4$ the set of belief states is $\cnSs_4= \langle\langle\pearlD_1,\pearlD_2,\pearlD_3\rangle,\pearlC\rangle ~|~  \pearlD_1,\pearlD_2,\pearlD_3,\pearlC \in E\}$, where $d_i$ is the contribution of node $\cnN_i$ to the diagnostic support of $\cnN_4$. 

For $\cnN_2$ the \emph{context} is the causal supports from above, $\cnCs_2\!=\!E$, while the \emph{observations} capture the influence of the ``H''-``A'' sensor, $\cnOs_2\!=\!\{\langle\pearlD\rangle ~|~  \pearlD\in E\}$. In contrast the observations for $\cnN_4$ need to capture the influence of the different child diagnostic supports, so $\cnOs_4 = \{\langle\pearlD_1,\pearlD_2,\pearlD_3\rangle ~|~  \pearlD_1,\pearlD_2,\pearlD_3 \in E\}$.

The observation update operators need to replace the diagnostic supports of the current belief with the observation, which is more complicated for $\cnN_4$ due to its multiple children, $\cnOUO_2(\{\vec{d_1},\vec{d_2},\vec{d_3}\}, \langle \vec{d},c\rangle) = \langle \Sigma_{i=1}^3 \vec{d_i},c\rangle$. Ignoring the influence of actions, the prediction update operator simply replaces the causal support of the current belief with the context from above, so $\cnAUO_2(\{c'\},\emptyset, \langle\langle\vec{d}\rangle,c\rangle) = \langle\langle\vec{d}\rangle,c'\rangle$.
\end{example*}

\subsection{Cognitive Hierarchy}

Nodes are interlinked in a hierarchy, where sensing data is passed up the \emph{abstraction hierarchy}, while actions and context are sent down the hierarchy (Figure~\ref{figNodeInteractions}).

\begin{figure}[ht]
	\centering
	\includegraphics[height=0.25\textwidth]{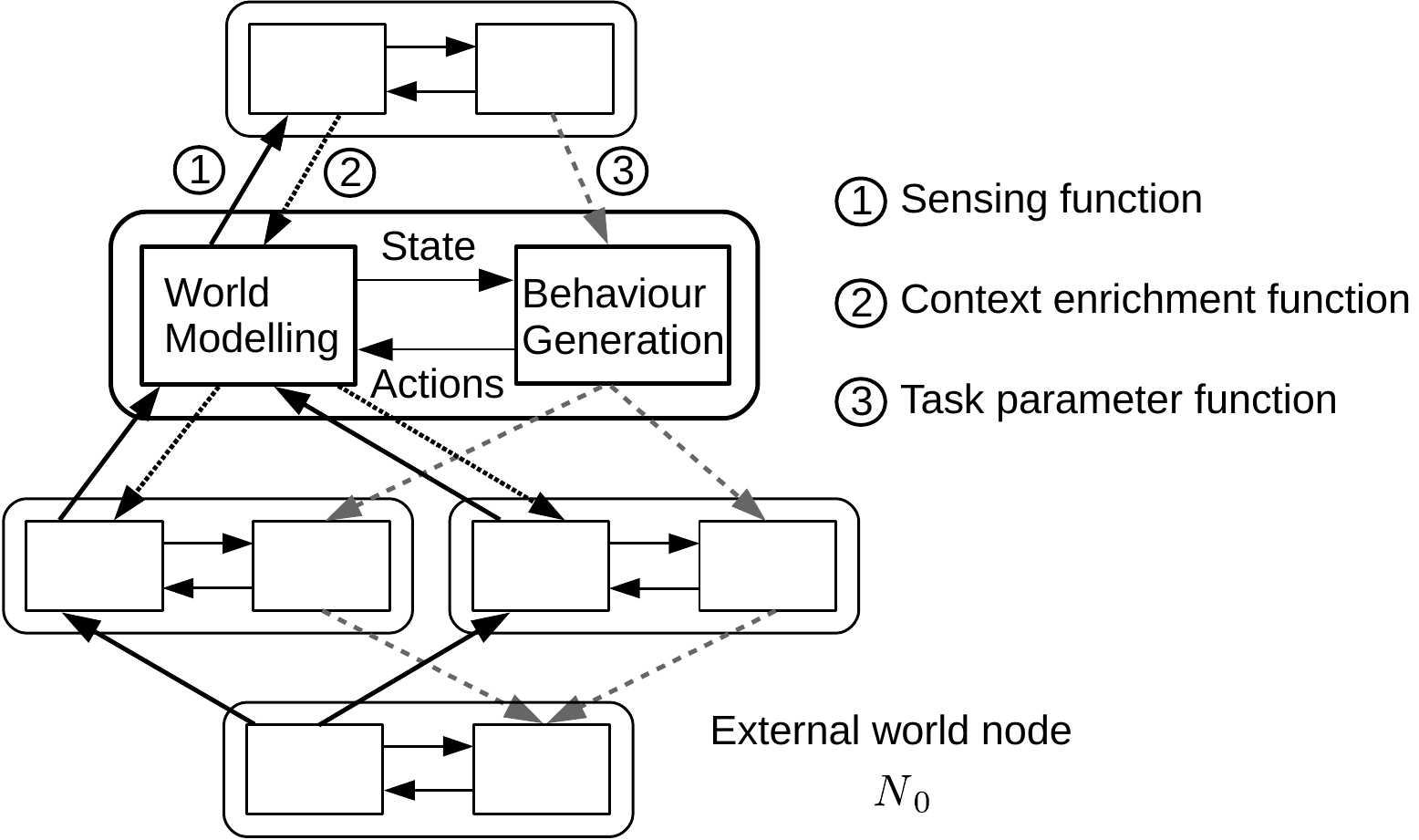}
	\caption{A cognitive hierarchy, highlighting internal interactions as well as the sensing, action, and context graphs.}
        \label{figNodeInteractions}
\end{figure}

\begin{definition}
\label{d:chierarchy}
A cognitive hierarchy is a tuple $\cnH = (\cnNs,\cnN_0,\cnFPs)$ s.t:
\begin{itemize}
\item $\cnNs$ is a set of cognitive nodes and $\cnN_0\in \cnNs$ is a distinguished node corresponding to the external world.
\item $\cnFPs$ is a set of function triples $\langle \cnSF_{i,j}, \cnTF_{j,i}, \cnCF_{j,i} \rangle \in \cnFPs$ that connect nodes $\cnN_i,\cnN_j\in\cnNs$ where:
  \begin{itemize}
    \item $\cnSF_{i,j}: \cnSs_i \rightarrow 2^{\smallmath{\cnOs_j}}$ is a sensing function, and
    \item $\cnTF_{j,i}: 2^{\smallmath{\cnAs_j}} \rightarrow 2^{\smallmath{\cnTs_i}}$ is a task parameter function.
    \item $\cnCF_{j,i}: \cnSs_j \rightarrow 2^{\smallmath{\cnCs_i}}$ is a context enrichment function.
  \end{itemize}
\item Sensing graph: each $\cnSF_{i,j}$ represents an edge from node $\cnN_i$ to $\cnN_j$ and forms a directed acyclic graph (DAG) with $\cnN_0$ as the unique source node of the graph.
\item Prediction graph: the set of task parameter functions (equivalently, the context enrichment functions) forms a converse to the sensing graph such that $\cnN_0$ is the unique sink node of the graph.
\end{itemize}
\end{definition}

Definition~\ref{d:chierarchy} differs from the original with the introduction of the \emph{context enrichment} functions and the naming of the \emph{prediction graph} (originally the \emph{action graph}). The connection between nodes consists of triples of sensing, task parameter and context functions. The sensing function extracts observations from a lower-level node in order to update a higher level node, while the context enrichment function performs the converse. The task parameter function translates a higher-level node's actions into a set of task parameters, which is then used to select the active policy for a node.

Finally, the external world is modelled as a distinguished node, $\cnN_0$. Sensing functions allow other nodes to observe properties of the external world, and task parameter functions allow actuator values to be modified, but $\cnN_0$ doesn't ``sense'' properties of other nodes, nor does it generate task parameters for those nodes. Similarly, context enrichment functions connected to $\cnN_0$ would simply return the empty set, unless one wanted to model unusual properties akin to the quantum effects of observations on the external world. Beyond this, the internal behaviour of $\cnN_0$ is considered to be opaque.

The running example can now be encoded formally as a cognitive hierarchy, again with the following showing only the salient features of the encoding.

\begin{example*}
We construct a hierarchy $\cnH\!=\!(\cnNs, \cnN_0, \cnFPs)$, with $\cnNs\!=\!\{ \cnN_0, \cnN_1, \ldots, \cnN_4\}$. The function triples in $\cnFPs$ will include $\cnSF_{0,2}$ for the visual sensing of the middle letter, and $\cnSF_{2,4}$ and $\cnCF_{4,2}$ for the sensing and context between $\cnN_2$ and $\cnN_4$.

The function $\cnSF_{0,2}$ returns the probability of the input being the characters ``H'' and ``A''. Here $\cnSF_{0,2}(\exAH) = \{\langle 0.5,0.5\rangle\}$.

Defining $\cnSF_{2,4}$ and $\cnCF_{4,2}$ requires a conditional probability matrix $M=
\begin{bmatrix}
  1 & 0\\
  0 & 1\\
\end{bmatrix}
$
to capture how the letters ``H'' and ``A'' contribute to the recognition of ``THE'' and ``CAT''.

For sensing from $\cnN_2$ we use zeroed vectors to prevent influence from the diagnostic support components from $\cnN_1$ and $\cnN_2$. Hence
 $\cnSF_{2,4}(\langle\langle\pearlD\rangle,\pearlC\rangle)\!=\{\langle\langle 0,0\rangle, \!\cnNorm\cdot M \cdot \pearlD^T,\langle 0,0\rangle\rangle\}$, where $\pearlD^T$ is the transpose of vector $\pearlD$, and $\cnNorm$ is a normalisation constant.

For context we capture how $\cnN_4$'s causal support and its diagnostic support components from $\cnN_1$ and $\cnN_2$ influences the causal support of $\cnN_2$. Note, that this also prevents any feedback from $\cnN_2$'s own diagnostic support to its causal support. So,
$\cnCF_{4,2}(\langle\langle\pearlD_1,\pearlD_2,\pearlD_3\rangle,\pearlC\rangle)\!=\!\{\cnNorm\cdot (\pearlD_1 \cdot \pearlD_3 \cdot \pearlC) \cdot M\}$.
\end{example*}

\subsection{Active Cognitive Hierarchy}

The above definitions capture the static aspects of a system but require additional details to model its operational behaviour. Note, the following definitions are unmodified from the original formalism and are presented here because they are necessary to the developments of later sections.

\begin{definition}
\label{d:activenode}
An active cognitive node is a tuple $\cnAN = (\cnN, \cnS, \cnP, \cnAss)$ where: 1) $\cnN$ is a cognitive node with $\cnSs$, $\cnPs$, and $\cnAs$ being its set of belief states, set of policies, and set of actions respectively, 2) $\cnS \in \cnSs$ is the current belief state, $\cnP \in \cnPs$ is the current policy, and $\cnAss \in 2^{\smallmath{\cnAs}}$ is the current set of actions.
\end{definition}

Essentially an active cognitive node couples a (static) cognitive node with some dynamic information; in particular the current belief state, policy and set of actions.

\begin{definition}
\label{d:activehierarchy}
An active cognitive hierarchy is a tuple $ \cnAH = (\cnH, \cnANs)$ where $\cnH$ is a cognitive hierarchy with set of cognitive nodes $\cnNs$ such that for each $\cnN \in \cnNs$ there is a corresponding active cognitive node $\cnAN=(\cnN, \cnS, \cnP, \cnAss) \in \cnANs$ and vice-versa.
\end{definition}

The active cognitive hierarchy captures the dynamic state of the system at a particular instance in time. Finally, an \emph{initial active cognitive hierarchy} is an active hierarchy where each node is initialised with the initial belief state and policy of the corresponding cognitive node, as well as an empty set of actions.

\subsection{Cognitive Process Model}
\label{s:processmodel}

The \emph{process model} defines how an active cognitive hierarchy evolves over time and consists of two steps. Firstly, sensing observations are passed up the hierarchy, progressively updating the belief state of each node. Next, task parameters and context are passed down the hierarchy updating the active policy, the actions, and the belief state of the nodes.

We do not present all definitions here, in particular we omit the definition of the \emph{sensing update} operator as this remains unchanged in our extension. Instead we define a \emph{prediction update} operator, replacing the original \emph{action update}, with the new operator incorporating both context and task parameters in its update.
First, we characterise the updating of the beliefs and actions for a single active cognitive node.

\begin{definition}
\label{d:actionNU}
Let $\cnAH\!=\!(\cnH, \cnANs)$ be an active cognitive hierarchy with $\cnH\!=\!(\cnNs, \cnN_0, \cnFPs)$. The prediction update of $\cnAH$ with respect to an active cognitive node $\cnAN_i\!=\!(\cnN_i, \cnS_i, \cnP_i, \cnAss_i) \in \cnANs$, written as $\cnAPNU(\cnAH,\cnAN_i)$ is an active cognitive hierarchy $\cnAH' = (\cnH, \cnANs')$ where
$\cnANs' = \cnANs\!\setminus\!\{\cnAN_i\} \cup \{\cnAN_i'\}$ and $\cnAN'_i = (\cnN_i, \cnAUO_i(C, \cnAss'_i,\cnS_i), \cnP'_i, \cnAss'_i)$ s.t:
\begin{itemize}
\item if there is no node $\cnN_x$ where $\langle \cnSF_{i,x}, \cnTF_{x,i}, \cnCF_{x,i} \rangle\!\in\! \cnFPs$ then: $\cnP'_i\!=\!\cnP_i, \cnAss'_i\!=\!\cnP_i(\cnS_i) \mbox{ and } C\!=\!\emptyset$,
\item else:\\
$\begin{array}[t]{l@{}l@{}l@{}l@{}l}
~~~& \cnP'_i &~=~ \cnPSF_i\!\!&\!\!(T) \mbox{ and } \cnAss'_i = \cnP'_i(\cnS_i), \\
   &       T &~=~ \bigcup  \{ & \cnTF_{x,i}(\cnAss_x)~|~\langle \cnSF_{i,x}, \cnTF_{x,i}, \cnCF_{x,i} \rangle\in\cnFPs \mbox{~where~}\\
   &         &             & ~\cnAN_x =(\cnN_x,\cnS_x, \cnP_x, \cnAss_x) \in \cnANs\}\\
   &       C &~=~ \bigcup  \{ & \cnCF_{x,i}(\cnS_x)~|~\langle \cnSF_{i,x}, \cnTF_{x,i}, \cnCF_{x,i} \rangle\in\cnFPs \mbox{~where~}\\
   &         &             & ~\cnAN_x =(\cnN_x,\cnS_x, \cnP_x, \cnAss_x) \in \cnANs\}\\
\end{array}$
\end{itemize}
\end{definition}

The intuition for Definition~\ref{d:actionNU} is straightforward. Given a cognitive hierarchy and a node to be updated, the update process returns an identical hierarchy except for the updated node. This node is updated by first selecting a new active policy based on the task parameters of all the connected higher-level nodes. The new active policy is applied to the existing belief state to generate a new set of actions. Both these actions and the context from the connected higher-level nodes are then used to update the node's belief state.

Using the single node update, updating the entire hierarchy simply involves successively updating all its nodes.

\begin{definition}
\label{d:actionMU}
Let $\cnAH = (\cnH, \cnANs)$ be an active cognitive hierarchy with $\cnH = (\cnNs, \cnN_0, \cnFPs)$ and $\cnGG$ be the prediction graph induced by the task parameter functions in $\cnFPs$. The action process update of $\cnAH$, written $\cnAPGU(\cnAH)$, is an active cognitive model:
\[
\cnAH' = \cnAPNU(\ldots\cnAPNU(\cnAH,\cnAN_n),\ldots\cnAN_0)
\]
where the sequence $[\cnAN_n,\ldots, \cnAN_0]$ consists of all active cognitive nodes of the set $\cnANs$ such that the sequence satisfies the partial ordering induced by the prediction graph $\cnGG$.
\end{definition}

Importantly, the update ordering in Definition~\ref{d:actionMU} satisfies the partial ordering induced by the prediction graph, thus guaranteeing that the prediction update is well-defined.

\begin{lemma}
\label{l:apu}
For any active cognitive hierarchy $\cnAH$ the prediction process update of $\cnAH$ is well-defined.
\end{lemma}
\begin{proof}
Follows from the DAG. \comment{ (see~\cite{clark2016framework})}
\end{proof}

The final part of the process model, which we omit here, is the combined operator, $\cnPGU$, that first performs a sensing update followed by a prediction update. This operation follows exactly the original and similarly the theorem that the process model is well-defined also follows.

We can now apply the update process (sensing then prediction) to show how it operates on the running example.

\begin{example*}
When $\cnN_2$ senses the symbol $\exAH$, $\cnSF_{0,2}$ returns that ``A'' and ``H'' are equally likely, so $\cnOUO_2$ updates the diagnostic support of $\cnN_2$ to $\langle\langle 0.5,0.5\rangle\rangle$. On the other hand $\cnN_1$ and $\cnN_2$ unambiguously sense ``C'' and ``T'' respectively, so $\cnN_4$'s observation update operator, $\cnOUO_4$, will update its diagnostic support components to $\langle\langle 0,1\rangle,\langle 0.5,0.5\rangle,\langle 0,1\rangle\rangle$. The nodes overall belief, $\langle 0,1\rangle$, is the normalised product of the diagnostic support components and the causal support, indicating here the unambiguous recognition of ``CAT''.

Next, during prediction update, context from $\cnN_4$ is passed back down to $\cnN_2$, through $\cnSF_{4,2}$ and $\cnAUO_2$, updating the causal support of $\cnN_2$ to $\langle 0,1\rangle$. Hence, $\cnN_2$ is left with the belief state $\langle\langle\langle  0.5,0.5\rangle\rangle,\langle 0,1\rangle\rangle$, which when combined, indicates that the symbol $\exAH$ should be interpreted as an ``A''.
\end{example*}

We next appeal to another simple example to illustrate the use of context to improve world modelling and in turn behaviour generation in a cognitive hierarchy.


\section{A Simple Visual Servoing Example}
Consider a mobile camera tasked to track an object sliding down a frictionless inclined plane. The controller is constructed as a three-node cognitive hierarchy. Figure \ref{figTracker} depicts the cognitive hierarchy and the scene.

\begin{figure}[ht]
	\centering
	\includegraphics[width=0.3\textwidth]{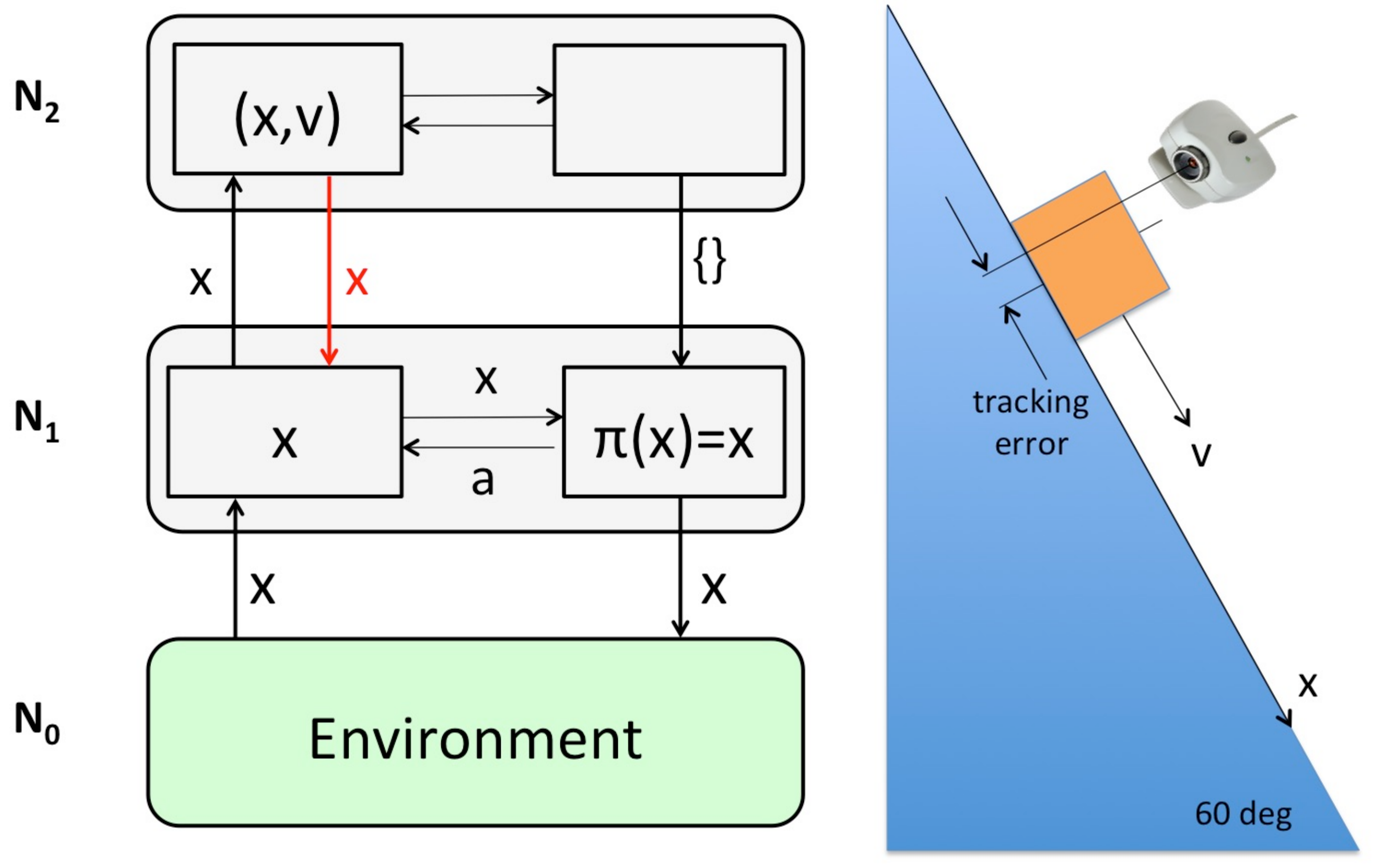}
	\caption{A three-node cognitive hierarchy controller tasked to visually follow an object. Context flow is shown in red.}
	\label{figTracker}
\end{figure} 

The performance of the controller will be determined by how well the camera keeps the object in the centre of its field-of-view, specifically the average error in the tracking distance over a time period of 3 seconds. 

The details of the instantiation of the cognitive hierarchy controller follow. The cognitive hierarchy is $\cnH\!=\!(\cnNs, \cnN_0, \cnFPs)$ with $\cnNs\!=\!\{ \cnN_0, \cnN_1, \cnN_2\}$. $\cnN_0$ is the unique opaque node representing the environment. The cognitive language for $\cnN_1$ is a tuple $\cnL_1\!\!\!=\!\!\!(\cnSs_1, \cnAs_1, \cnTs_1, \cnOs_1, \cnCs_1)$, and for $\cnN_2$ it is $\cnL_2\!\!\!=\!\!\!(\cnSs_2, \cnAs_2, \cnTs_2, \cnOs_2, \cnCs_2)$. The cognitive nodes are $\cnN_1 = (\cnL_1, \cnPs_1, \cnPSF_1, \cnOUO_1,  \cnAUO_1, \cnS^0_1, \cnP^0_1)$ and $\cnN_2 = (\cnL_2, \cnPs_2, \cnPSF_2, \cnOUO_2,  \cnAUO_2, \cnS^0_2, \cnP^0_2)$. For brevity we only describe the material functions. 

The belief state of $\cnN_1$ is the position of the object: $\cnSs_1\!=\!\{x~|~x \!\in\!{\mathbb R}\}$. The belief state of $\cnN_2$ is both the position and velocity of the object: $\cnSs_2\!=\!\{\langle x,v \rangle~|~x,v\!\in\!{\mathbb R}\}$. The object starts at rest on the inclined plane at the origin: $\cnS^0_1=0.0$ and 
 $\cnS^0_2=\langle 0.0,0.0 \rangle$.

$\cnN_1$ receives object position observations from the environment: $\cnOs_1=\{x~|~x \!\in\!{\mathbb R}\}$. These measurements are simulated from the physical properties of the scene and include a noise component to represent errors in the sensor measurements: $\cnSF_{0,1}(\cdot) = \{0.5kt^2+\nu\}$, with constant acceleration $k=8.49$ $m/s^2$, $t$ the elapsed time and $\nu$ zero mean Gaussian random noise with a standard deviation of $0.1$. The acceleration assumes an inclined plane of 60 degrees in a $9.8$ $m/s^2$ gravitational field.  The $\cnN_1$ observation update operator implements a Kalman filter with a fixed gain of $0.25$: $\cnOUO_1(\langle \{x\},y \rangle)=(1.0-0.25)y+0.25x$. 

$\cnN_2$ receives observations $\cnOs_2=\{x~|~x \!\in\!{\mathbb R}\}$ from $\cnN_1$: $\cnSF_{1,2}(x) = \{x\}$. In turn it updates its position estimate accepting the value from $\cnN_1$: $\cnOUO_2(\langle \{x\},\langle y,v \rangle \rangle)=\langle x,v \rangle $. The prediction update operator uses a physics model to estimate the new position and velocity of the object after time-step $\delta t = 0.05$ seconds: $\cnAUO_2(\langle \{\},\{\},\langle x,v \rangle \rangle)=\langle x+v\delta t+ 0.5k\delta t^2,v+k\delta t \rangle$ with known acceleration $k=8.49$.

Both $\cnN_1$ and $\cnN_2$ have one policy function each. The $\cnN_2$ policy selects the $\cnN_1$ policy. The effect of the $\cnN_1$ policy: $\pi_1(x)=\{x\}$, is to move the camera to the estimated position of the object via the task parameter function connecting the environment: $\cnTF_{1,0}(\{x\}) = \{x\}$.

We consider two versions of the $\cnN_1$ prediction update operator. Without context the next state is the commanded policy action: $\cnAUO_1(\langle \{x\},\{y\},z \rangle)=y$. With context the context enrichment function passes the $\cnN_2$ estimate of the position of the object to $\cnN_1$:  $\cnCF_{2,1}(\langle x,v \rangle)=\{x\}$, where $\cnCs_1=\{x~|~x \!\in\!{\mathbb R}\}$. The update operator becomes: $\cnAUO_1(\langle \{x\},\{y\},z \rangle)=x$.

When we simulate the dynamics and the repeated update of the cognitive hierarchy at $1/\delta t$ Hertz for 3 seconds, we find that without context the average tracking error is $2.004\pm 0.009$. Using context the average tracking error reduces to $0.125\pm 0.015$---a 94\% error reduction.\footnote{It is of course intuitive in this simple example that as $\cnN_2$ has the benefit of the knowledge of the transition dynamics of the object it can better estimate its position and provide this context to direct the camera.}

\section{Using Context to Track Objects Visually}
Object tracking has application in augmented reality, visual servoing, and man-machine interfaces. We consider the problem of on-line monocular model-based tracking of multiple objects without markers or texture,  using the 2D RGB camera built into the hand of a Baxter robot.  The use of natural object features makes this a challenging problem.

Current practice for tackling this problem is to use 3D knowledge in the form of a CAD model, from which to generate a set of edge points (control points) for the object~\cite{Lepetit:2005:MMT:1166405.1166406} . The idea is to track the corresponding 2D camera image points of the visible 3D control points as the object moves relatively to the camera. The new pose of the object relative to the camera is found by minimising the perspective re-projection error between the control points and their corresponding 2D image.

However, when multiple objects are tracked, independent CAD models fail to handle object occlusion. In place of the CAD models we use the machinery provided by a 3D physics simulator. The object-scene and virtual cameras from a simulator are ideal to model the higher level context for vision. We now describe how this approach is instantiated as a cognitive hierarchy with contextual feedback. It is important to note that the use of the physics simulator is not to replace the real-world, but is used as mental imagery efficiently representing the spatial belief state of the robot.

\subsection{Cognitive Hierarchy for Visual Tracking}
We focus on world-modelling in a two-node cognitive hierarchy (Figure \ref{fig2NodeCM}). The external world node that includes the Baxter robot, streams the camera pose and RGB images as sensory input to the arm node. The arm node belief state $\cnS=\{p^a\} \cup \{\langle p_a^i, c^i \rangle | $object$\;i \}$, where $p^a$ is the arm pose, and for all recognised objects $i$ in the field of view of the arm camera, $p_a^i$ is the object pose relative to the arm camera, and $c^i$ is the set of object edge lines and their depth.  The objects in this case include scattered cubes on a table. Information from the arm node is sent to the spatial node that employs a Gazebo physics simulator as mental imagery to model the objects.

\begin{figure}[ht]
	\centering
	\includegraphics[width=0.47\textwidth]{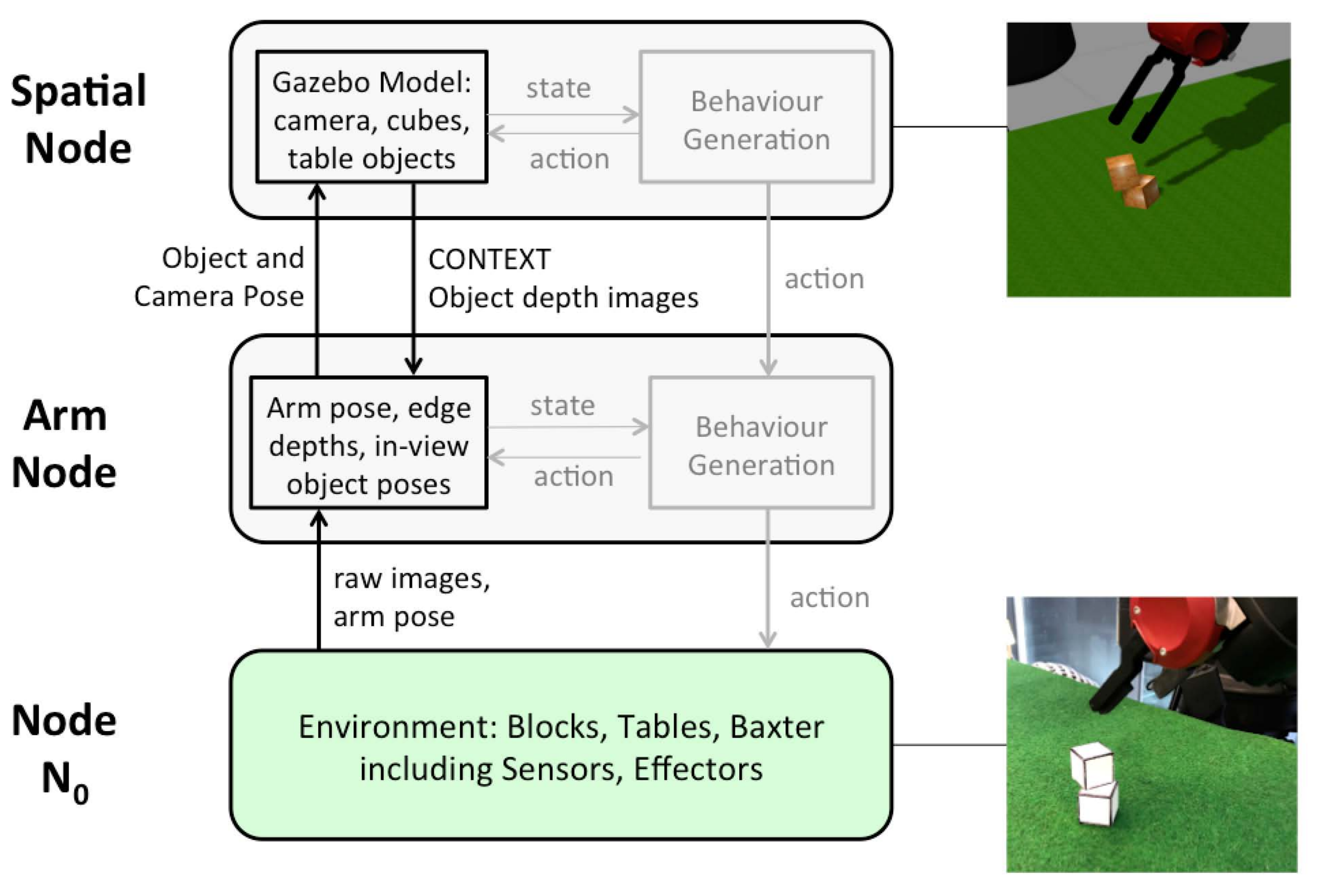}
	\caption{Cognitive hierarchy comprising an arm node and a spatial node. Context from the spatial node is in the form of an object segmented depth image from a simulated special camera that shadows the real camera.}
	\label{fig2NodeCM}
\end{figure} 

A novel feature of the spatial node is that it simulates the robot's arm camera as an object aware depth camera. No such camera exist reality, but the Gazebo spatial belief state of the robot is able to not only provide a depth image, but one that segments the depth image by object. This object aware depth image provides the context to the arm node to generate the required control points. 

\subsection{Update Functions and Process Update}
We now describe the update functions and a single cycle of the process update for this cognitive hierarchy.

The real monocular RGB arm camera is simulated in Gazebo with an object aware depth camera with identical characteristics (i.e. the same intrinsic camera matrix). The simulated camera then produces depth and an object segmentation images from the simulated objects that corresponds to the actual camera image. This vital contextual information is then used for correcting the pose of the visible objects.

\comment{
This contextual information is useful for correcting the pose of visible objects as described below. The context enrichment function from the spatial node to the arm node will be described later in the update cycle.
}

The process update starts with the sensing function $\cnSF_{N_0,Arm}$ that takes the raw camera image and observes all edges in the image, represented as a set of line segments, $l$. 
\[ \cnSF_{N_0,Arm}(\{rawImage\}) = \{l\} \]

The observation update operator $\cnOUO_{Arm}$ takes the expected edge lines $c^i$ for each object $i$ and transforms the lines to best match the image edge lines $l$ \cite{Lepetit:2005:MMT:1166405.1166406}. The update function uses the OpenCV function \emph{solvePnP} to find a corrected pose $p^i_a$ for each object $i$ relative to the arm-camera $a$ \footnote{The pose of a rigid object in 3D space has 6 degrees of freedom, three describing its translated position, and three the rotation or orientation, relative to a reference pose.}.
\[ \cnOUO_{Arm}(\{ l, c^i  | \mbox{object}\;i\} ) = \{p_a^i | \mbox{object}\;i\}  \]

The sensing function from the arm to spatial node takes the corrected pose $p_a^i$ for each object $i$, relative to the camera frame $a$, and transforms it into the Gazebo reference frame via the Baxter's reference frame given the camera pose $p^a$.
\[ \cnSF_{Arm,Spatial}(\{p^a,  \langle p_a^i, c^i \rangle | \mbox{object}\;i \}) = \{g_a^i | \mbox{object}\;i\} \]

The spatial node observation update $\cnOUO_{Spatial}$, updates the pose of all viewed objects $g_a^i$ in the Gazebo physics simulator. Note $\{g_a^i | \mbox{object}\; i \} \subset \mbox{gazebo state}$.
\[ \cnOUO_{Spatial}(\{g_a^i | \mbox{object}\;i\}) = \mbox{gazebo.move}(i, g_a^i) \;\; \forall i  \]

The update cycle now proceeds down the hierarchy with prediction updates. The prediction update for the spatial node $\cnAUO_{Spatial}$ consists of predicting the interaction of objects in the simulator under gravity. Noise introduced during the observation update may result in objects separating due to detected collisions or settling under gravity.
\[ \cnAUO_{Spatial}( \mbox{gazebo state}) = \mbox{gazebo.simulate}(\mbox{gazebo state})) \]

We now turn to the context enrichment function $\cnCF_{Spatial,Arm}$ that extracts predicted camera image edge lines and depth data for each object in view of the simulator. 
\[ \cnCF_{Spatial,Arm}(\mbox{gazebo state})  = \{c^i | \mbox{object}\;i\} \]
\begin{figure}[ht]
	\centering
	\includegraphics[width=0.48\textwidth]{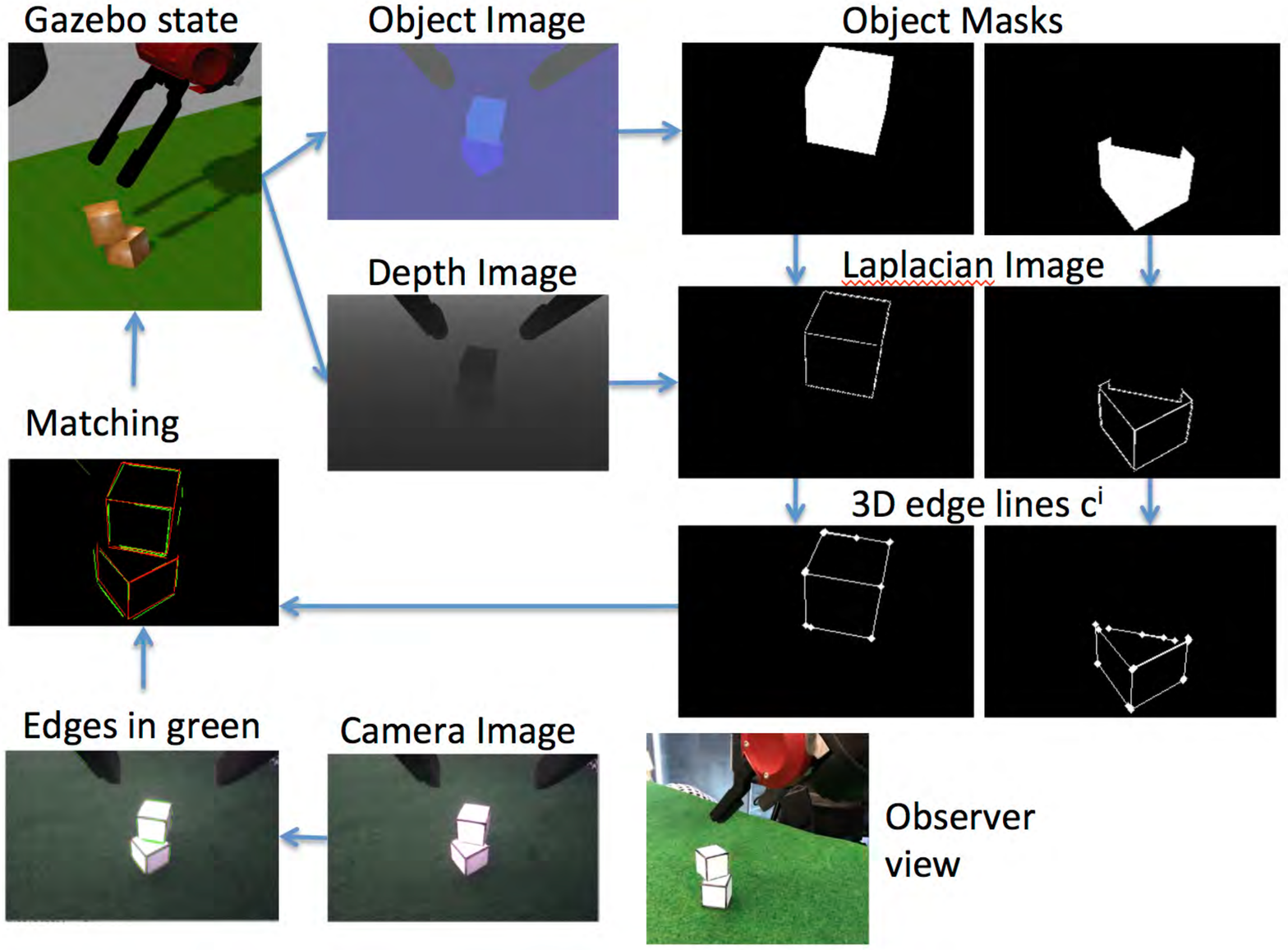}
	\caption{The Process update showing stages of the context enrichment function and the matching of contextual information to the real camera to correct the arm and spatial node belief state.}
	\label{figContext}
\end{figure} 
The stages of the context enrichment function $\cnCF_{Spatial,Arm}$ are shown in Figure \ref{figContext}. 
The simulated depth camera extracts an object image that identifies the object seen at every pixel location.  It also extracts a depth image that gives the depth from the camera of every pixel. The object image is used to mask out each object in turn. Applying a Laplacian function to the part of the depth image masked out by the object yields all visible edges of the object. A Hough line transform identifies line end points in the Laplacian image and finds the depth of their endpoints from the depth image, producing $c^i$. 
\begin{figure}[ht]
	\centering
	\includegraphics[width=0.47\textwidth]{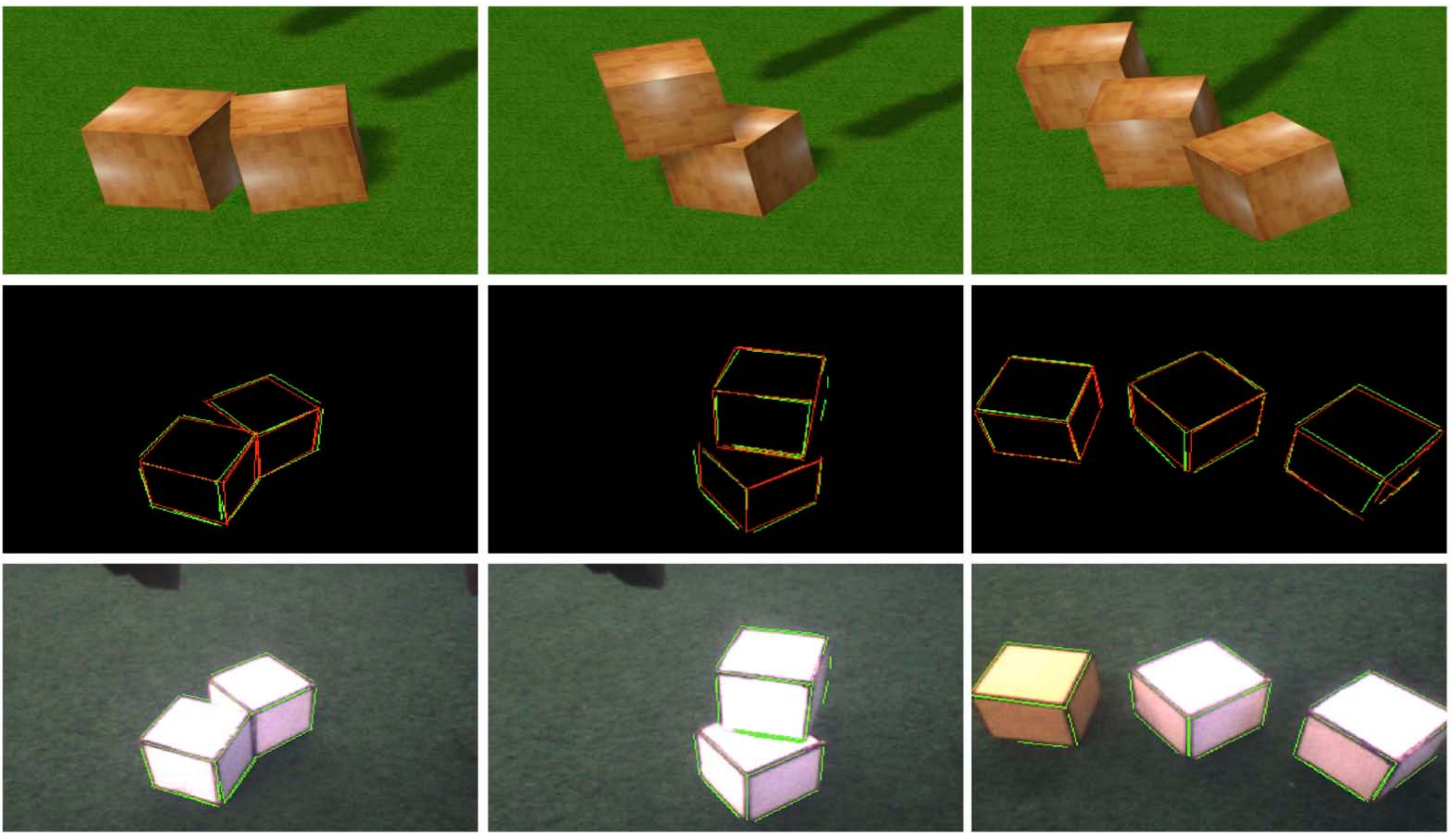}
	\caption{Tracking several cube configurations. Top row: Gazebo GUI showing spatial node state. 2nd row: matching real image edges in green to simulated image edges in red. Bottom row: camera image overlaid with edges in green.}
	\label{figCorrespondence}
\end{figure} 

Figure \ref{figCorrespondence} shows the cognitive hierarchy tracking several different cube configurations. This is only possible given the context from the spatial belief state. Keeping track of the pose of objects allows behaviours to be generated that for example pick up a cube with appropriately oriented grippers.


\section{Related Support and Conclusion}
There is considerable evidence supporting the existence and usefulness of top-down contextual information. Reliability \cite{Biederman1981On-the-semantic} and speed \cite{cavanagh91whatsup} of scene analysis provides early evidence. 

These observations are further supported by neuroscience, suggesting that feedback pathways from higher more abstract processing areas of the brain down to areas closer to the sensors are greater than those transmitting information upwards \cite{hawkins04on}. The authors summarise the process - ``what is actually happening flows up, and what you expect to happen flows down". \citeauthor{Gilbert2013Top-down-influe}~(\citeyear{Gilbert2013Top-down-influe}) argue that the traditional idea that the processing of visual information consists of a sequence of feedforward operations needs to be supplemented by top-down contextual influences.

In the field of robotics, recent work in online interactive perception shows the benefit of predicted measurements from one level being passed to the next-lower level as state predictions \cite{DBLP:conf/iros/MartinB14}.

This paper has included and formalised the essential element of context in the meta framework of cognitive hierarchies. The process model of an active cognitive hierarchy has been revised to include context updates satisfying the partial order induced by the prediction graph. We have illustrated the role of context with two simple examples and a novel way to track the pose of texture-less objects with a single 2D camera. As a by-product contribution we prove that general belief propagation in causal trees can be embedded into our framework testifying to its representation versatility. 

%



\begin{appendices}
\section{Causal Networks  as Cognitive Hierarchies}

\newcommand{\plTree}{\mbox{$\mathcal{T}$}}
\newcommand{\plProc}{\mbox{${P}$}}
\newcommand{\plProcs}{\mbox{${\mathcal{P}}$}}
\newcommand{\plProp}{\mathtext{Prop}}
\newcommand{\plBel}{\mathtext{BEL}}
\newcommand{\cnBel}{\mathtext{BEL}}
\newcommand{\plReal}{\mbox{${\mathbb R}$}}

The motivating example highlighted the use of context in a cognitive hierarchy inspired by belief propagation in causal trees. In this appendix we extend the example to the general result that any Bayesian causal tree can be encoded as a cognitive hierarchy. We do this by constructively showing how to encode a causal tree as a cognitive hierarchy and proving the correctness of this method with respect to propagating changes through the tree.

Pearl describes a causal tree as a set of \emph{processors} where the connection between processors is explicitly represented within the processors themselves. Each processor maintains \emph{diagnostic} and \emph{causal} supports, as well as maintaining a conditional probability matrix for translating to the representation of higher-level processors. The description of the operational behaviour of causal trees is presented throughout Chapter~4 (and summarised in Figure~4.15) of ~\citeauthor{pearl88probabilistic}~(\citeyear{pearl88probabilistic}).

The cognitive hierarchies introduced here are concerned with robotic systems and consequently maintain an explicit notion of sensing over time. In contrast causal networks are less precise about external inputs and changes over time. As a bridge, we model that each processor has a diagnostic support component that can be set externally.  Finally, note that we adopt the convenience notation $\cnEF$ to represent a function of arbitrary arity that always returns the empty set.

\begin{definition}
\label{d:transform}
Let $\{\plProc_1,\ldots\plProc_n\}$ be a causal tree. We construct a corresponding cognitive hierarchy $\cnH = (\{\cnN_0, \cnN_1,\ldots, \cnN_n\}, \cnN_0, \cnFPs)$ as follows:
\begin{itemize}

\item For processor $\plProc_i$  with $m$ children, and diagnostic and causal supports $d,c\in{\mathbb R}^n$, define $\cnSs_i\!=\!\{\langle\langle d_E,d_1,\ldots, d_m \rangle, c'\rangle | d_E,d_1,\ldots,d_m,c'\in{\mathbb R}^n\}$, with initial belief state $\cnS_i\!=\!\langle\langle d,\ldots,d\rangle, c\rangle$. Define $\cnOs_i\!=\!\{\langle d_E, d_1,\ldots, d_m\rangle | d_E,d_1,\ldots,d_m\in{\mathbb R}^n\}$ and $\cnCs_i\!=\!\mathbb R^n$.

\item For processor $\plProc_i$ with corresponding cognitive node $\cnN_i$, define $\cnOUO_i(o, \langle \vec{d}, c\rangle) = \langle\Sigma_{\vec{d'}\in o}\vec{d'}, c\rangle$, and $\cnAUO_i(\{c'\}, \emptyset , \langle \vec{d}, c\rangle)=\langle \vec{d},c'\rangle$.

\item For each pair of processors $\plProc_i$ and $\plProc_j$, where $\plProc_j$ is the $k$-th child of $\plProc_i$'s $m$ children (from processor subscript numbering), and $M_j$ is the conditional probability matrix of $\plProc_j$, then define a triple $\langle \cnSF_{j,i}, \cnEF, \cnCF_{i,j} \rangle \in \cnFPs$ s.t:

\begin{itemize}

\item $\cnSF_{j,i}(\langle \vec{d},c\rangle) = \{\langle d_E, d_1,\ldots,d_m\rangle\}$, where $d_{h\not=k}$ are zeroed vectors and $d_k = \cnNorm\cdot M_j \cdot (\prod_{d'\in \vec{d}}d')^T$.

\item $\cnCF_{i,j}(\langle \langle d_E, d_1, \ldots,d_m \rangle, c \rangle) = \{c'\}$, such that $c' = \cnNorm\cdot( \prod_{h\not=k}d_h \cdot c)\cdot M_j$.

\item where $\cnNorm$ is a normalisation constant for the respective vectors, and $x^T$ is the transpose of vector $x$.

\end{itemize}

\item For processors $\plProc_i$, with diagnostic support $d\in\plReal^n$, define a triple $\langle \cnSF_{0,i}, \cnEF, \cnEF \rangle \in \cnFPs$ where $\cnSF_{0,i}(\cnS_0\in\cnSs_0) = \{\langle d_E, d_Z,\ldots,d_Z\rangle\}$, where $d_Z$ is a zeroed vector and $d_E \in \plReal^n$ is the external input of $\plProc_i$.

\end{itemize}
\end{definition}

While notationally dense, Definition~\ref{d:transform} is simply a generalisation of the construction used in the running example and is a direct encoding of causal trees. Note, this construction could be further extended to poly-trees, which Pearl also considers, but would require a more complex encoding.

To establish the correctness of this transformation we can compare how the structures evolve with sensing. The belief measure of a processor is captured as the normalised product of the diagnostic and causal supports, $\plBel(P_i) = \cnNorm\cdot d_i \cdot c_i$. However, for a cognitive node the diagnostic support needs to be computed from its components. Hence, given the belief state $\langle \langle d_E,d1, \ldots,d_m \rangle, c \rangle$ of an active node $\cnAN_i$ with $m$ children, we can compute the belief as $\cnBel(\cnAN_i) = \cnNorm\cdot \prod_{j=1}^m d_j \cdot c$.

\begin{lemma}
Given a causal tree $\{\plProc_1,\ldots\plProc_n\}$ and a corresponding cognitive hierarchy $\cnH$ constructed via Definition~\ref{d:transform}, then the causal tree and the initial active cognitive hierarchy corresponding to $\cnH$ share the same belief.
\end{lemma}
\begin{proof}
By inspection, $\plBel(\plProc_i) = \cnBel(\cnAN_i)$ for each $i$.
\end{proof}

Now, we establish that propagating changes through an active cognitive hierarchy is consistent with propagating beliefs through a causal tree. We abuse notation here to express the overall belief of a casual tree (resp. active cognitive hierarchy) as simply the beliefs of its processors (resp. nodes).

\begin{theorem}
\label{t:equiv}
Let $\plTree$ be a causal tree and $\cnAH$ be the corresponding active cognitive hierarchy constructed via Definition~\ref{d:transform}, such that $\plBel(\plTree) = \cnBel(\cnAH)$. Then for any changes to the external diagnostic supports of the processors and corresponding changes to the sensing inputs for the active cognitive hierarchy, $\plBel(\plProp(\plTree)) = \cnBel(\cnPGU(\cnAH))$.
\end{theorem}
\begin{proof}
Pearl establishes that changes propagated through a causal tree converge with a single pass up and down the tree. Any such pass satisfies the partial ordering for the cognitive hierarchy process model. Hence the proof involves the iterative application of the process model, confirming at each step that the beliefs of the processors and nodes align.
\end{proof}

Theorem~\ref{t:equiv} establishes that Bayesian causal trees can be captured as cognitive hierarchies. This highlights the significance of extending cognitive hierarchies to include context, allowing for a richer set of potential applications.

\end{appendices}

\newpage

\section*{Acknowledgments}
This material is based upon work supported by the Asian Office of Aerospace Research and Development (AOARD) under Award No: FA2386-15-1-0005. This research was also supported under Australian Research Council's (ARC) {\em Discovery Projects\/} funding scheme (project number~DP\,150103035). Michael Thielscher is also affiliated with the University of Western Sydney.

\section*{Disclaimer}
Any opinions, findings, and conclusions or recommendations expressed in this publication are those of the authors and do not necessarily reflect the views of the AOARD.

\bibliographystyle{aaai}
\bibliography{paper}

\end{document}